\documentclass[12pt]{article}
\usepackage{fullpage}

\usepackage[color]{changebar}
\usepackage{amsfonts}
\usepackage{amsmath,amssymb}
\usepackage{bm}
\usepackage{mathtools}
\usepackage{graphicx}
\usepackage[ruled,linesnumbered,vlined]{algorithm2e}
\usepackage{multirow}
\usepackage{mathptmx}      
\usepackage{algorithm2e}
\usepackage{gnuplottex}
\usepackage{xspace}
\usepackage{makecell}
\usepackage{doi}

\usepackage{epstopdf} 
\usepackage{verbatim}
\usepackage{caption,subcaption}
\usepackage{amsthm}
\usepackage{xspace} 
\newcommand{\Schwefel}{\textsc{Schwefel}\xspace}
\newcommand{\Rosenbrock}{\textsc{Rosenbrock}\xspace}

\newcommand{\Elliptic}{H. C. \textsc{Elliptic}\xspace}
\newcommand{\Rastrigin}{\textsc{Rastrigin}\xspace}
\newcommand{\Sphere}{\textsc{Sphere}\xspace}

\usepackage{tabularx,booktabs}

\usepackage{wrapfig}
\newcolumntype{Y}{>{\centering\arraybackslash}X}
\newtheorem{defi}{Definition}
\newtheorem{sat}{Theorem}

\newtheorem{lem}{Lemma}

\newcommand{\En}{\Phi}

\renewcommand{\P}{{\rm P}}
\newcommand{\argmin}{\operatornamewithlimits{argmin}}
\newcommand{\cond}[1][]{\;#1|\;}
\providecommand{\e}[1]{\ensuremath{\!\times\! 10^{#1}}}
\newcommand\B{\rule[-2ex]{0pt}{0pt}}        
\newcommand{\GminusX}{\Delta}

\graphicspath{{./Figs/}}

\allowdisplaybreaks 

\begin{document}

\title{Self-adaptive Potential-based Stopping Criteria for Particle
  Swarm Optimization
}

\author{Bernd Bassimir
\thanks{Corresponding author} \quad Manuel Schmitt \quad Rolf Wanka
\\
Department of Computer Science\\
University of Erlangen-Nuremberg, Germany\\
{\sf\small $\{$bernd.bassimir, manuel.schmitt, rolf.wanka$\}$@fau.de}}

\date{}

\maketitle

\begin{abstract}
We study the variant of Particle Swarm Optimization (PSO) that applies random velocities in a dimension instead of the regular velocity update equations as soon as the so-called \emph{potential} of the swarm falls below a certain bound in this dimension, arbitrarily set by the user. In this case, the swarm performs a \emph{forced move}. In this paper, we are interested in how, by counting the forced moves, the swarm can decide for itself to stop its movement because it is improbable to find better solution candidates as it already has found. We formally prove that when the swarm is close to a (local) optimum, it behaves like a blind-searching cloud, and that the frequency of forced moves exceeds a certain, objective function-independent value. Based on this observation, we define stopping criteria and evaluate them experimentally showing that good solution candidates can be found much faster than applying other criteria.
\end{abstract}


\paragraph*{Declarations of interest} None.

\section{Introduction}
\label{sec:intro}

\paragraph{Background}
Particle Swarm Optimization (PSO) is a meta-heuristic for so-called continuous black box optimization problems, which means that the objective function is not
explicitly known in form of a, e.\,g., closed formula.
PSO produces good results in a variety of different real world applications.
The classical PSO as first introduced by Eberhart and
Kennedy~\cite{ken_eb_1995,eb_ken_1995}
in the year 1995 works in (solution candidates improving) iterations and
can be very easily implemented and
adapted to the users' applications, which lead to increased
attention not only among computer scientists.
In order to further improve the performance,
many authors present changes to the original, ``plain,'' or classical
PSO scheme (for exact definitions, see Sec.~\ref{sec:definitions})
to improve the quality of the returned solution.

A serious problem of PSO that can be sometimes observed is the
phenomenon of premature stagnation, i.\,e., the convergence of the swarm
to a non-optimal solution.
This phenomenon has been theoretically addressed by
Lehre and Witt in~\cite{LW:11}.
To overcome such stagnation, the authors propose \emph{Noisy PSO} 
that adds a ``noise'' term to the velocity at every move.
They prove that for the Noisy PSO 
started on a certain simple $1$-dimensional objective
function the first hitting 
time of the $\delta$-neighborhood of the \emph{global} optimum is finite.
However, as proved in~\cite{SW:13}, 
premature stagnation of classical PSO does not occur at all
when the search space
is $1$-dimensional and the objective function is continuous, i.\,e., in the $1$-dimensional case, PSO
provably finds at least a local optimum, almost surely (in the
well defined sense of probability theory).
Furthermore, \cite{SW:13} shows a similar result for a
slightly modified PSO in the
general $D$-dimensional case
(for stagnation-related results regarding the unmodified
PSO, see~\cite{RSW:15}).
This slightly modified PSO assigns a small random velocity in solely
one dimension only if the so-called \emph{potential} of the
swarm -- a fundamental, measurable quantity of the swarm -- of all
particles in this dimension falls below a certain (arbitrary and small)
bound $\delta$.
In the following, we call such random velocity
particle moves \emph{forced steps} and the PSO variant f-PSO.
The f-PSO provably finds a local optimum almost
surely~\cite{SW:13}.
Monitoring the swarm's potential and increasing it from time to time
by the forced steps
is the key ingredient to mathematically proving successful
convergence to a (local) optimum.
In this paper, we will use the forced steps of f-PSO and the
potential for overcoming another problem of heuristics, the question
when to stop the algorithm.

\paragraph{Problem and new contribution}
An important problem that arises in the context of PSO and other
iterative optimization methods (see \cite[p.\,23f]{SW:81} and further
papers, referenced in the related work section below)
is \emph{when} to stop the iterations
and to return the best found admissible solution.
Usually, the process is terminated (i) when an upper limit
on the number of iterations is reached,
or (ii) when an upper limit on the number of evaluations of the
objective function is reached, or (iii) when
the chance of achieving significant improvements
in further iterations is extremely low.
The choice of the mentioned two upper limits obviously depends
on the concrete objective function which means that the
user has to have detailed knowledge of the objective
function and to ``intervene'' by hand.
In (iii), it is desired and advantageous that the algorithm adaptively
decides when to stop, so external intervention is not necessary
anymore.
To achieve this for PSO, many criteria were introduced in the literature.
For a short overview, see below.
A commonly used criterion is the swarm diameter, i.\,e., the maximum distance between 
two particles, as a measure for the expansion and thus the movement capability of the swarm.
When PSO is extended with forced steps, this and also other criteria of this kind do not work because the expansion
of the overall swarm is forcefully kept above a certain value and can no longer converge like it does for the classical PSO algorithm. 
Just the globally best position found by the swarm converges to optimum.

The goal of this paper is to characterize
the behavior of the f-PSO when it is close to a (local) optimum.
Our experiments and mathematical investigations show that the number of forced moves does not
only increase significantly when the distance to the next
local optimum falls below a certain bound, but that
additionally the number of forced moves performed close to a local optimum
is independent of the objective function.
In particular, we prove by potential arguments that the swarm
``pulsates'' in a cloud around the best solution candidate found so far.
Note that a similar behavior has been described recently
in~\cite{YSG:18} for the classical PSO.
Therefore, by measuring the frequency of occurrences of forced moves,
this frequency can act as a stopping criterion, so the swarm may come to
a self-determined halt.
All findings are experimentally supported.

\paragraph{Related work}
Adaptive stopping criteria for PSO
have been investigated by Zielinski et al. \cite{ZPL:05} and
by Zielinski and Laur \cite{ZL:07}.
In these papers, a list of upper limit-based and adaptive termination
criteria is presented and in experiments applied to ($2$-dimensional)
established benchmark functions and a real-world problem, resp.
In \cite{KHLF:07}, Kwok et al. introduce a stopping criterion for PSO
based on the rate of improvements found by the swarm
in a given time interval. Quality of the
termination is assured using the non-parametric sign-test
enforcing a low false-positive rate.
A further stopping approach due to Ong and Fukushimah~\cite{OF:15}
combines PSO with gene matrices.

To name some work beyond PSO, Ribeiro et al.~\cite{RRS:11}
introduce a stopping criterion that stops the
GRASP (Greedy Randomized Adaptive Search Procedures)
algorithm when the probability of an improvement is below a
threshold under an experimentally fitted normal distribution.

Safe et al.~\cite{SCPB:04} present a study of various aspects
associated with the specification of termination conditions
for simple genetic algorithms.
In \cite{AK:00}, Aytug and Koehler introduce a stopping criterion that
stops a genetic algorithm when the optimal solution is found with a
specified confidence and thus no real further progress can be expected.

In a very general, ground breaking investigation, Solis and Wets~\cite{SW:81}
consider random search in general and also address the
question of stopping criteria.
In this context, Dorea~\cite{D:90} presents two adaptive stopping criteria.

Stopping criteria in the context of multi-objective optimization is
presented by Mart{\'\i} et al. \cite{MGBM:16}.

\paragraph{Organization of paper}
Sec.~\ref{sec:definitions} presents the necessary description of the classical and
 forced PSO (f-PSO) and the definition of the potential of a swarm.
In Sec.~\ref{sec:fm}, we experimentally and mathematically analyze
the behavior of f-PSO when it comes close to a (local) optimum.
In Sec.~\ref{sec:term}, based on the mathematical analysis in Sec.~\ref{sec:fm},
we present the new termination criteria and show their practicability
by experimental evaluations.


\section{Definitions}
\label{sec:definitions}
Due to the wide variety of existing PSO variants, we first state the exact ``classical'' PSO algorithm on which our work is based.

\begin{defi}[Classical PSO Algorithm]\label{cPSO}
A \emph{swarm} $\cal S$ of $N$ particles moves through the $D$-dimens-ional search space
$\mathbb{R}^D$ with $\mathcal{D}=\{1,\ldots,D\}$ being the set of dimensions. Each particle $n\in{\cal S}$ consists of 
a \emph{position} $X^n\in\mathbb{R}^D$, a \emph{velocity} $V^n\in\mathbb{R}^D$ and a \emph{local attractor} $L^n\in\mathbb{R}^D$, storing the best position particle $n$ has visited so far. 
Additionally, the particles of the swarm share
information via the \emph{global attractor} $G\in\mathbb{R}^D$, describing the best 
point \emph{any} particle has visited so far, i.\,e.,
as soon as a particle has performed its move\footnote{the particles' moves are executed sequentially, so there
is some arbitrary order of the particles.}, it possibly updates the global attractor immediately.

The actual movement of the swarm is governed 
by the following \emph{movement equations} where 
$\chi$, $c_1$, $c_2\in\mathbb{R}^+$
are some positive constants to be fixed later, and $r$ and $s$ are drawn u.\,a.\,r. from $[0,1]^D$
every time the equation is applied.
\begin{align*}
V^n &:= \chi\!\cdot\! V^n + c_1\!\cdot\! r\odot (L^n-X^n) + c_2\!\cdot\! s\odot (G-X^n)\\
X^n &:= X^n+V^n\nonumber
\end{align*}
Here, $\odot$ denotes entrywise multiplication (Hadamard product).
The application of the equations on particle $n$ is called the
\emph{move} of $n$.
When all particles have executed their moves,
the swarm has executed one \emph{iteration}.
\end{defi}
Now we repeat the definition of a swarm's potential
measuring how close it is to convergence, i.\,e., we describe a measure for its movement. 
A swarm with high potential should be more likely to reach search points far away from the current
global attractor, while the potential of a converging swarm approaches $0$.
These considerations lead to the following definition~\cite{SW:13a} (the original
definition in~\cite{SW:13} is more complex just for
technical reasons):

\begin{defi}[Potential]\label{Energy}
Fix a moment of the computation of swarm $\cal S$.
For $d\in \mathcal{D}$, the current \emph{potential} $\En_d$
of $\cal S$ in dimension $d$ is
$$
\En_d:=\sum_{n=1}^N 
(\underbrace{|V_d^n|\allowbreak +|G_d-X_d^n|}_{=:\ \phi_d^{n}})
=\sum_{n=1}^N  \phi_d^{n}
$$
$\Phi=(\Phi_1,\ldots,\Phi_D)$ is the total potential of $\cal S$, and
$\phi_d^{n}$ is the \emph{contribution} of particle $n$ to the
potential of $\cal S$ in dimension~$d$.
\end{defi}

Note that the
potential of the swarm
has an entry
in every dimension.
The swarm comes to a halt if $\Phi \to (0,\ldots,0)$.
So, if the particles come close to $G$, the swarm may stop even if $G$ is
a non-optimal point, an incident that is called (premature) stagnation.
The single values of $\Phi$ can be actually computed and, hence,
be used for decisions on the swarm.
Between two different dimensions, the potential difference might be large,
and ``transferring''
potential from one dimension to another is not possible due to the
movement equations.
On the other hand, along the same dimension the particles
influence each other and can transfer potential from one particle
to the other.
This is the reason why there is no potential of individual particles,
but only their contribution to the potential.

To address the phenomenon of stagnation, \cite{SW:15} slightly
modified the PSO movement equations
from Definition~\ref{cPSO} as follows by ``recharging'' potential and, hence,
keeping the swarm moving furthermore.

\begin{defi}[f-PSO]\label{modified}
The modified movement of the swarm is governed 
by the 
following movement equations
where 
$\chi$, $c_1$, $c_2$, $\delta\in\mathbb{R}^+$ are some positive constants
to be fixed later, $r$, $s$ and $t$ are drawn u.\,a.\,r. from $[0,1]$ for every move and dimension of a particle.
\begin{align*}
V_d^n &:= 
\begin{cases}
  (2\cdot t - 1)  \cdot \delta, & \hspace{-9.15em} \text{if } \forall n' \in\mathcal{S}: |V_d^{n'}| + |G_d - X_d^{n'}| < \delta \text{\ \emph{[forced velocity update]}}\\
  \chi\cdot V_d^n + c_1\!\cdot\! r\cdot (L_d^n-X_d^n)
  +c_2\!\cdot\! s\cdot (G_d-X_d^n), & otherwise \label{modVelup} \hfill\text{\quad\emph{[usual, regular velocity update]}}
\end{cases}\\
X^n &:= X^n+V^n.\nonumber
\end{align*}
If the forced velocity update
applies to a particle, we call its move and the corresponding dimension in the move \emph{forced}.
An iteration of the swarm is called \emph{forced} if during this iteration at least one particle performs a forced move. The whole method is called
\emph{f-PSO}.
\end{defi}

If it is necessary to identify the values at the beginning of
iteration~$i$, we write $X_d^{n,i}$, $V_d^{n,i}$ etc.

\begin{algorithm}[H]
\caption{\label{alg:modPSO}\quad f-PSO}

\SetKwInOut{Input}{input}
\SetKwInOut{Output}{output}

\Input{Objective function $f:\mathbb{R}^D\to\mathbb{R}$, number $N$ of particles}
\Output{$G\in\mathbb{R}^D$}
\For{$n=1 \to N$}{
  Initialize $X^n$ randomly\;
  Initialize $V^n$ with $\vec{0}$\;
  Initialize $L^n := X^n$\;}
  Initialize $G := \argmin\limits_{\{L^n\mid n \in \mathcal{S}\}}f$\;
\Repeat{termination criterion met \label{line-sixteen}\tcp*[h]{One can use our new criteria developed in Sec.~\ref{sec:term} below}}{\tcp{Iterations; $t$, $r$, $s$
drawn u.\,a.\,r.\ from $[0,1]$ every time}
  \For{$n=1 \to N$}{
    \For{$d=1 \to D$}{\tcp{We use $\delta=10^{-7}$  in our experiments}
      \eIf{$\forall n' \in\mathcal{S}: |V_d^{n'}| + |G_d - X_d^{n'}| < \delta$\label{line-nine}}{\tcp{Execute a forced velocity update}
        $V_d^n:=(2\cdot t - 1)  \cdot \delta$ \label{line-ten}\tcp*[l]{hence, $V_d^n\in[-\delta,\delta]$}
        }{\tcp{Execute the usual, regular velocity update}
        \tcp{We use $\chi = 0.72984$, $c_1= 1.49617$, $c_2=1.49617$ in our experiments}
        $V_d^n:=\chi\cdot V_d^n + c_1\cdot r\cdot (L_d^n-X_d^n) + c_2 \cdot s\cdot (G_d-X_d^n)$\;
      }
      $X_d^n := X_d^n + V_d^n$\;
    }
    \lIf{$f(X^n)\le f(L^n)$}{
      $L^n := X^n$
    }
    \lIf{$f(X^n)\le f(G)$}{
      $G := X^n$
    }
  }
}
\textbf{return} $G$\;
\end{algorithm}

Algorithm~\ref{alg:modPSO} provides a formal and detailed overview over f-PSO.
The introduction of forced velocity updates guarantees that
the swarm (or more precisely, the global attractor $G$) almost surely
does not converge to a non-optimal point, but
finds a local optimum~\cite{SW:13}. In our analysis and for the experiments,
we used the common parameter settings $\chi = 0.72984$, $c_1= 1.49617$, $c_2=1.49617$ and $N>1$ as suggested and used in \cite{CK:02},
which are parameter settings that are widely used in the literature.


\section{Behavior of the f-PSO algorithm}
\label{sec:fm}

The idea 
of the modification of the classical PSO algorithm is to help the
swarm overcome
(premature) stagnation.
The modification is implemented
in the calculation of the new velocity during a move
of a particle (see Def.~\ref{modified} and Lines~\ref{line-nine}
and~\ref{line-ten} of Algorithm~\ref{alg:modPSO}).
The new velocity of a particle in the current dimension $d$ is drawn u.\,a.\,r.
from ${[-\delta, \delta]}$ when for all particles $n'\in\mathcal{S}$
their contributions  $\phi_d^{n'}$ to
the potential in dimension $d$ is less than $\delta$, which
means
that the range in which the swarm can optimize is small in dimension $d$.

The main topic of this section is to examine if
at such a situation the particle swarm uses from now on only forced moves or, more desirable, the swarm will recover
and continue using regular
velocity updates. In \cite{SW:15}, 
it is shown
that
indeed the latter is the case, unless
the global attractor $G$ is already in the neighborhood of a local optimum.
As in this case almost all moves are forced, we will see in
Sec.~\ref{sec:term} that this
can be translated into a criterion to stop the PSO's execution.

\subsection{Experiments}
We first introduce the notion of the forcing frequency to quantify the number of applications of the forced velocity update (Line~\ref{line-ten} of Algorithm~\ref{alg:modPSO}).

\begin{defi}[Absolute and relative forcing frequency]\label{sigma}
  Let $I$ denote a (time) interval of $|I|$ iterations during a run of the \emph{f-PSO} algorithm.
\begin{itemize}
\item[(a)]
  Let $\sigma(I,d)$ denote the number of times
  forced velocity updates (Line~\ref{line-ten}) have been
  executed in interval $I$ for dimension $d$, $d\in\{1,\ldots,D\}$.
  $\sigma(I,d)$ is called the \emph{absolute forcing frequency} in dimension $d$
  over interval $I$.
\item[(b)]
  $\sigma(I)=\sum_{d=1}^D\sigma(I,d)$ is the total number
  of forced velocity updates (Line~\ref{line-ten})
  in interval $I$ counted over all
  dimensions.
  $\sigma(I)$ is called the \emph{absolute forcing frequency} over
  interval $I$.
\end{itemize}
Analogously, the \emph{relative forcing frequency} is $\sigma(I,d)/|I|$ and
$\sigma(I)/|I|$, resp.
\end{defi}

To explore the behavior of the particle swarm with respect to
the absolute forcing frequency $\sigma(I,d)$,
we performed a series of experiments on the well known benchmark functions
\Schwefel, \Rosenbrock, \Rastrigin, \Elliptic  and \Sphere
(for a comprehensive overview of these and many other benchmark functions,
see~\cite[Sec~4.2]{helwig:10}).
\Sphere and \Elliptic will be analyzed in detail in Sec.~\ref{subsec:ExpSphere}.
The tests were performed with Ra\ss{}' HiPPSO~\cite{HiPPSO}, a high precision implementation of PSO,
in order to rule out the influence of insufficient computer systems'
precision when applying a forced velocity update
with $\delta$
close to the precision of a, e.\,g., \texttt{long double} variable in C++.
The tests were performed with $N=3$ particles, $D=30$ dimensions, $\delta=10^{-7}$ and the
well-known swarm parameters already mentioned in Sec.~\ref{sec:definitions}.

\paragraph{\Schwefel function}

Fig. \ref{fig:opt_steps_schw} presents the measurements
obtained when optimizing the \Schwefel function with $D=30$.
The absolute forcing frequencies per dimension over the intervals $I_i=[0\ldots 50\,000\cdot i]$ and the
function values $f(G)$ of the current global attractors $G$ at the end of each interval are
depicted.
One can see that the absolute forcing frequency (gradient in the figure) is
relatively small if $f(G)$ is far away from the (unique)
optimum value and increases considerably when $f(G)$ approaches
the optimum. Additionally, one can see that the gradient of the absolute forcing frequency,
i.\,e., the relative forcing frequency, tends to be constant for all $i \geq i_0$ after some value $i_0$.
This is a showcase and remains
true for most of the other tested benchmark functions.
\begin{figure}
  \centering
    \scalebox{.9}{\input{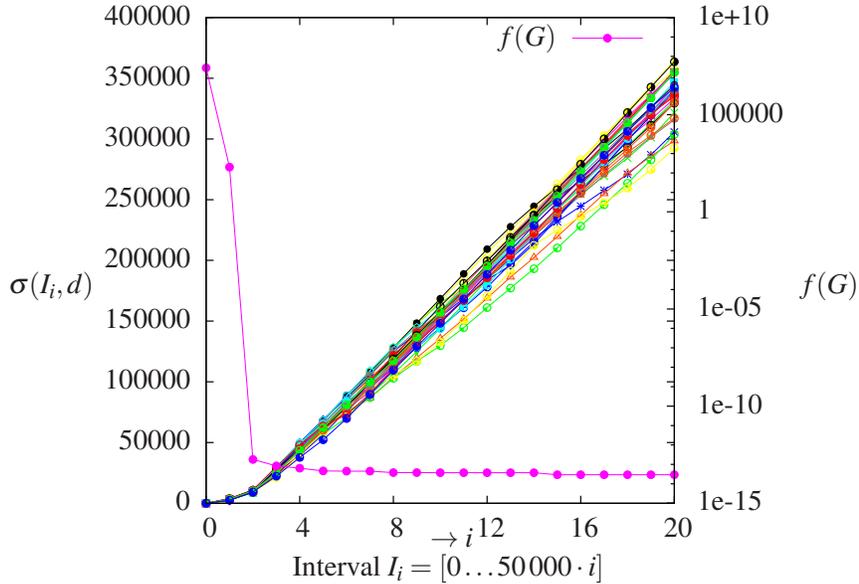}}
  \caption{Optimizing the \Schwefel function:
    development of the absolute forcing frequencies $\sigma(I_i,d)$,
    $d\in\{1,\ldots,30\}$, over the intervals
    $I_i=[0\ldots 50\,000\cdot i]$ (bundle of lines)
    compared to the development of
    the function value $f(G)$ of the global best position $G$
    after $50\,000\cdot i$ iterations (single red line).
    $N=3$ particles and $\delta=10^{-7}$ were used.}
  \label{fig:opt_steps_schw}
\end{figure}

\paragraph{\Rosenbrock function}

During the tests on the standard benchmark functions,
processing the \Rosenbrock function showed
a slightly different behavior.
Metaphorically speaking, there is a small ``banana-shaped''
valley in this function
that leads from a local optimum to the global optimum.
If the attractors are in this valley (what they are quite early),
the chance to improve
in some dimensions is quite small and
improvements can easily be voided
by setbacks in other dimensions.
If this happens, forced velocity updates will
be applied in some dimensions while in the other dimensions
the particles use the usual, regular velocity updates,
which can be seen clearly in Fig.~\ref{fig:opt_steps_ros}.
The chance to get moving into the optimum's direction again
is not zero, but quite small which leads to a long phase of
near-stagnation as also can be seen in Fig.~\ref{fig:opt_steps_ros}.
Eventually the swarm will re-start moving again, but it may
take a long time to do so.

\paragraph{\Rastrigin function}

Another phenomenon may occur on functions like \Rastrigin, where there
are many local optima.
Here, it might happen that the global attractor $G$ is already
close to a (good) local optimum, but the local attractor $L^n$
of some particle $n$ is close to a worse local optimum. 
In such a situation, the potential is governed by the distance between
$G$ and $L^n$, and the condition for executing a forced velocity update
is not satisfied, even though a good local near-optimum value $f(G)$
has already been found.
But our experiments showed that this happens hardly ever.
\begin{figure}
  \centering
  \scalebox{.9}{\input{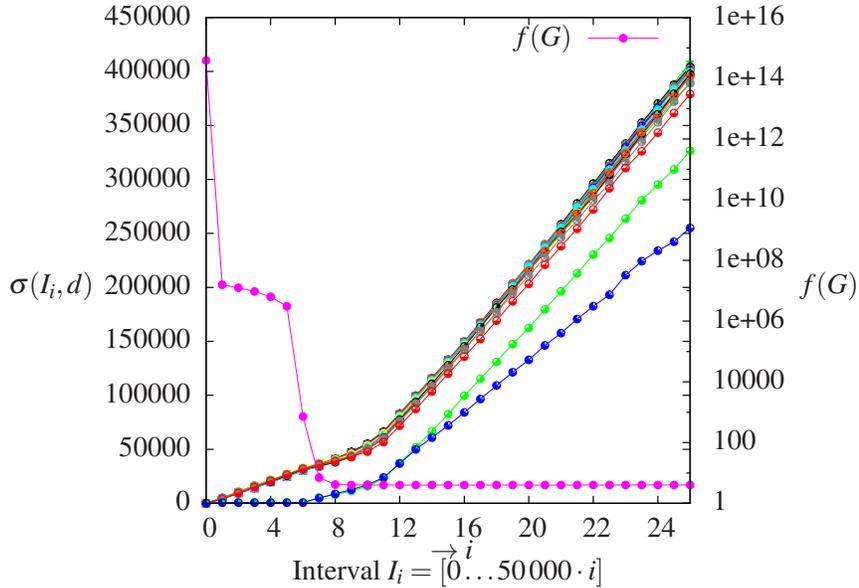}}
  \caption{Optimizing the \Rosenbrock function: 
    development of the absolute forcing frequencies $\sigma(I_i,d)$,
    $d\in\{1,\ldots,30\}$, over the intervals
    $I_i=[0\ldots 50\,000\cdot i]$ 
    compared to the development of
    the function value $f(G)$ of the global best position $G$
    after $50\,000\cdot i$ iterations (single red line).
    $N=3$ particles and $\delta=10^{-7}$ were used.
    Noticeable is the constant behavior of $\sigma(I_i,d)$ with $f(G)$
    far away from the optimal
    value of $0$, and the separation
    into forcing dimensions and non-forcing dimensions.
  }
  \label{fig:opt_steps_ros}
\end{figure}

\subsection{Experiments on the \Sphere Function, Phases, and
Theoretical Analysis}
\label{subsec:ExpSphere}

We now analyze the swarm's final behavior with the help
of experiments on \Sphere (the results for \Elliptic are similar).
These experiments show that the final behavior
can be classified into two main phases:
the approaching phase and the pulsation phase that in turn
has three distinct sub-phases: forced phase, lockout phase, recovery phase.
These findings hold for all objective functions, as the swarm
is in this final phase a kind of a blind searching cloud.

So, when the global attractor reaches the $\delta$-neighborhood of a local
optimum, the number of global attractor updates decreases and the
distance the global attractor moves becomes small.
At that time, the whole swarm will start to approach
the global attractor.
At some point during this \emph{approaching
phase}, the local attractors will be close to the global attractor and the
swarm will loose much of its potential. As described in \cite{SW:15},
the introduction of forced velocity updates
avoids that the actual positions of the particles 
approach each other
closer than a value of about $\delta$
and the swarm enters a kind of \emph{pulsation phase}:
periodically it contracts and expands.
However, the global attractor still converges (in the mathematical
``infinite time'' sense)
to the local optimum.
Also, the positions of the local attractors may further contract.

It can be observed that when the attractors are all almost at the same position and the particles converge to this position, the number of dimensions that are forced begins to increase.
During the pulsation
phase almost all particles are forced and the relative forcing frequency $\sigma(I,d)/|I|$ begins to
stagnate at a certain value $\partial\sigma_d$ (the gradient in the figures).
Hence, also $\sigma(I)/|I|$ begins to stagnate at $\partial\sigma$.
When this stagnation is reached, the f-PSO will become a periodic pulsating
process and will behave almost like a blind search in a box with side-length
of order of magnitude $\delta$ around the global attractor.
Hence, the objective function $f$ will become irrelevant and the value $\partial\sigma$
is independent of $f$.

\subsubsection{Experimental Identification of parameters influencing
$\partial\sigma$}

In order to identify
the parameters that influence 
$\partial\sigma$, a series
of experiments were performed with the \Sphere function. To study a pure pulsation
phase, the global and local attractors were initially all set to $\smash{\vec{0}}$, which is the global optimum of this function. With the swarm parameters $\chi$, $c_1$ and $c_2$ being fixed, the two crucial
remaining parameters that might influence
$\partial\sigma$ are
\begin{itemize}
\item the number $D$ of dimensions and
\item the number $N$ of particles,
\end{itemize}
whereas the choice of $\delta$ has no influence at all. By changing $\delta$ the range of $\phi_d$ causing a forced step is influenced for a dimension $d\in\mathcal{D}$. However, the random velocity assigned 
by the forced step in this dimension $d$ is changed in the same magnitude. Therefore the change of $\delta$ has no influence on the absolute forcing frequency as these two effects cancel each other.

Fig. \ref{fig:forced_dimensions_dimension} and \ref{fig:forced_dimensions_particle} show the effect of changes in the dimension number $D$ and swarm size $N$, resp., on the stagnation value $\partial\sigma$ for
fixed $|I|$.
\begin{figure}
  \centering
  \scalebox{.9}{\input{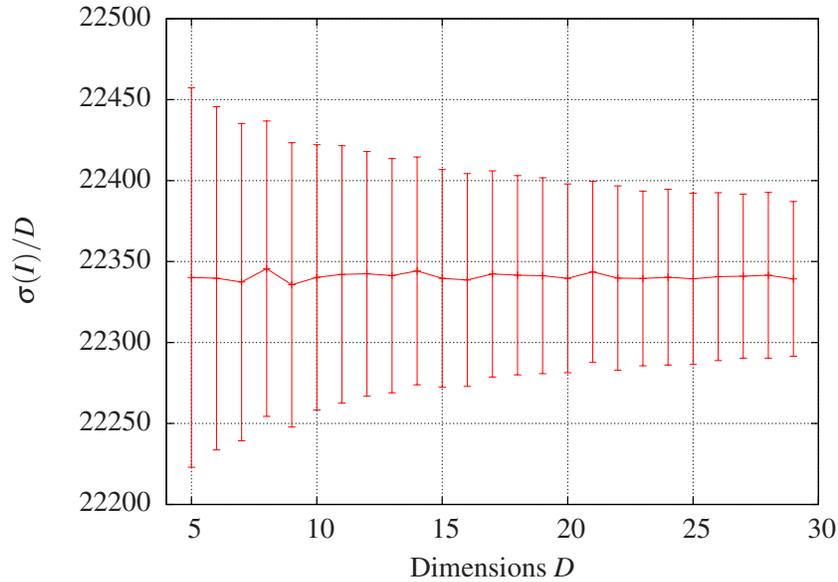}}
  \caption{Optimizing the \Sphere function: Absolute forcing frequency $\sigma(I)$ over intervals $I$ of length $|I|=50\,000$ relative to the number $D$ of dimensions at the global optimum
with $N=5$ particles and $\delta=10^{-7}$, varying the number $D$ of dimensions.
Shown are the average value and standard deviation of $100$ trials taking the average values of $10$ intervals per trial.}
  \label{fig:forced_dimensions_dimension}
\end{figure}
\begin{figure}
  \centering
  \scalebox{.9}{\input{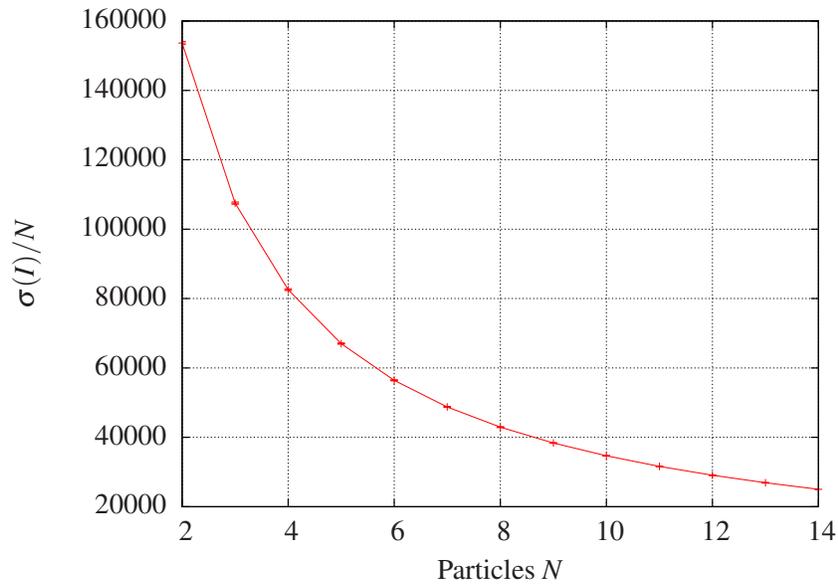}}
  \caption{Optimizing the \Sphere function: Absolute forcing frequency $\sigma(I)$ over intervals $I$ of length $|I|=50\,000$ relative to the number $N$ of particles at the global optimum
with $D=15$ dimensions and $\delta=10^{-7}$, varying the number of particles.
Shown are the average value and (the very small, almost invisible)
standard deviation of $100$ trials taking the average values of $10$ intervals per trial.}
  \label{fig:forced_dimensions_particle}
\end{figure}
\begin{figure}
  \centering
  \scalebox{.9}{\input{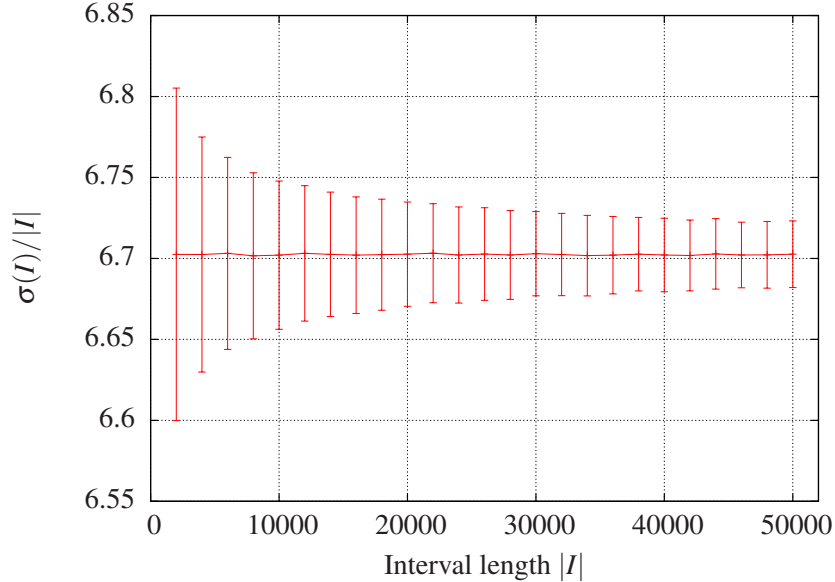}}
  \caption{Absolute forcing frequency $\sigma(I)$ over intervals $I$ relative to the length $|I|$ at the global optimum of the \Sphere function with $D=15$ dimensions and $\delta=10^{-7}$ and $N=5$ particles varying the length $|I|$ of the measured intervals. Shown are the average value and standard deviation of $100$ trials taking the average values of $10$ intervals per trial.}
  \label{fig:forced_dimensions_interval}
\end{figure}
The use of a forced move in a given dimension depends only on the
velocities
and positions
of all particles relative to the global attractor $G$ in this dimension. The dimensions are independent of each other, therefore the 
stagnation value (recall that $|I|$ is fixed)
for each number $D$ of dimensions does not change with an
increased number of dimensions as shown in Fig. \ref{fig:forced_dimensions_dimension}. The same
is not true for varying the number of particles, see Fig. \ref{fig:forced_dimensions_particle}.
If the length $|I|$ of the interval and the number $D$ of dimensions are fixed, increasing the number of
particles
increases the number of applications of the movement equations by $|I|\cdot D$ for each new particle. Therefore, more dimensions can be forced in a given interval, but there are also more particles that must have a partial potential $\phi$ less than $\delta$.

\subsubsection{Mathematical analysis of the sub-phases for
arbitrary objective functions}
We can identify three sub-phases during the pulsation phase which are repeated in a fixed order for a given dimension.
\begin{itemize}
\item[(i)]
Starting with a \emph{forced phase}, where each particle has its move in this dimension forced,
\item[(ii)] followed by a \emph{lockout phase} with the probability of a forced move being zero and
\item[(iii)]
finally a \emph{recovery phase} in which the particles
attempt to
converge to the global attractor until one move of a particle gets forced again and the cycle repeats itself.
\end{itemize}
To fully understand why this leads to a less
forced frequency $\sigma(I)$ per particle, see Fig.~\ref{fig:forced_dimensions_particle}, we have to take a closer look at the three sub-phases of the pulsation phase. As described in Def. \ref{modified}, a dimension $d\in\mathcal{D}$ is forced
during the move of a particle when all particles $n'$ including the particle itself have contribution $\phi_d^{n'}$ less than $\delta$
in this dimension~$d$.
For  $i \in \mathbb{N}$, $n \in \mathcal{S}$, and $d \in \mathcal{D}$
let $ F_{i,n,d}$  be the $\{0,1\}$-indicator variable being $1$ iff in iteration $i$ and dimension $d$ particle~$n$ has the velocity 
update forced.
The global and the local attractors are constant and in particular, it can be assumed that $G_d^i = G$ for fixed $G$.
We get the following observations regarding a forced move in a given dimension.

The first phase we analyze is the lockout phase.
Recall that $\phi_d^{n,i}=|V_d^{n,i}| + |G - X_d^{n,i}|$
is the contribution of particle $n$ to the potential of $\cal S$
in dimension $d$ in iteration $i$.

\begin{lem}
  \label{lem:stop}
  If $\phi_d^{n,i+1} 
      \geq \delta$ for particle $n$, then
  $\forall n'\in\mathcal{S},n'>n: F_{i,n',d} = 0$ and
  $\forall n'\in\mathcal{S}, n'\leq n:F_{i+1,n',d} = 0$.
\end{lem}

\begin{proof}
  With the attractors being constant, the only way for particle $n$ to have
  contribution $\phi_d^n$ 
  less than $\delta$
  is to make a swarm move that reduces the
  sum of the velocity and the distance to the global attractor to less than $\delta$.
  The earliest time this is possible is in the next iteration.
\end{proof}

By Lemma \ref{lem:stop}, we get the length of the lockout phase
as $N$, the number of particles.

For the forced phase, we have to look at the probability that
for a particle $n$ with a forced move in dimension $d$,
$\phi_d^n<\delta$ still applies in the next iteration.
\begin{lem}
  \label{lem:prob}
  Let $n\in \mathcal{S}, d\in\mathcal{D}$. Then for $n<N:\P [F_{i,n+1,d} = 1 \cond F_{i,n,d} = 1] = \frac{1}{2}$ and \\$\P [F_{i+1,1,d} = 1 \cond F_{i,N,d} = 1] = \frac{1}{2}$
\end{lem}

\begin{proof}
As each dimension is independent and the same calculations are performed
in each dimension, we may restrict our analysis to one dimension $d$
and omit the index $d$ in the following proofs.

  W.\,l.\,o.\,g., let $n=1$.
Let $\GminusX=G-X^{1,i}$ the (signed) distance of particle $1$ to the
fixed global attractor. Hence, $\phi^{1,i}=|V^{1,i}|+|\GminusX|$.
 \begin{align*}
\MoveEqLeft[1] \P \left[F_{i,2} = 1 \cond F_{i,1} = 1\right] 
    =\P \left[\phi^{1,i-1} < \delta \cond[\middle] F_{i,1} = 1\right]\\
    &\begin{alignedat}{2}
      =\,\,\,\,  &\P \left[\phi^{1,i+1} < \delta \cond[\middle]\GminusX \geq 0  \wedge F_{i,1} = 1\right]
                 &&\cdot \P \left[  \GminusX \geq 0 \cond[\middle] F_{i,1} = 1\right]\\
       + \,       &\P \left[\phi^{1,i+1} < \delta       \cond[\middle]\GminusX < 0                                  \wedge F_{i,1} = 1\right] &&\cdot \P \left[  \GminusX< 0                                            \cond[\middle] F_{i,1} = 1\right]\B\\
     \end{alignedat}\\
    &\begin{alignedat}{2}
       =\,\,\,\,  &\P \left[\phi^{1,i+1} < \delta       \cond[\middle](V^{1,i+1} \geq \GminusX \geq 0)            \wedge F_{i,1} = 1\right] &&\cdot \P \left[  V^{1,i+1}     \geq \GminusX       \geq 0           \cond[\middle] F_{i,1} = 1\right]\\
       + \,       &\P \left[\phi^{1,i+1} < \delta       \cond[\middle](\GminusX\geq V^{1,i+1} \geq 0)            \wedge F_{i,1} = 1\right] &&\cdot \P \left[  \GminusX   \geq V^{1,i+1}         \geq 0           \cond[\middle] F_{i,1} = 1\right]\\
       + \,       &\P \left[\phi^{1,i+1} < \delta       \cond[\middle](\GminusX \geq 0 \geq V^{1,i+1} )           \wedge F_{i,1} = 1\right] &&\cdot \P \left[  \GminusX   \geq 0 \geq V^{1,i+1}                   \cond[\middle] F_{i,1} = 1\right]\\
       + \,       &\P \left[\phi^{1,i+1} < \delta       \cond[\middle](\GminusX \leq V^{1,i+1} < 0)               \wedge F_{i,1} = 1\right] &&\cdot \P \left[  \GminusX   \leq V^{1,i+1}         < 0              \cond[\middle] F_{i,1} = 1\right]\\
       + \,       &\P \left[\phi^{1,i+1} < \delta       \cond[\middle](V^{1,i+1} \leq \GminusX < 0)               \wedge F_{i,1} = 1\right] &&\cdot \P \left[  V^{1,i+1}     \leq \GminusX       < 0              \cond[\middle] F_{i,1} = 1\right]\\
       + \,       &\P \left[\phi^{1,i+1} < \delta       \cond[\middle](\GminusX \leq 0 \leq V^{1,i+1})            \wedge F_{i,1} = 1\right] &&\cdot \P \left[  \GminusX   \leq 0 \leq V^{1,i+1}                   \cond[\middle] F_{i,1} = 1\right]\B\\
     \end{alignedat}\\
    &\begin{alignedat}{2}
       =\,\,\,\,  &\P \left[2 \cdot V^{1,i+1} - \GminusX < \delta   \cond[\middle](V^{1,i+1} \geq \GminusX \geq 0)            \wedge F_{i,1} = 1\right] &&\cdot \P \left[  V^{1,i+1}     \geq \GminusX       \geq 0           \cond[\middle] F_{i,1} = 1\right]\\
       + \,       &\P \left[\GminusX < \delta                           \cond[\middle](\GminusX \geq V^{1,i+1} \geq 0)            \wedge F_{i,1} = 1\right] &&\cdot \P \left[  \GminusX   \geq V^{1,i+1}         \geq 0           \cond[\middle] F_{i,1} = 1\right]\\
       + \,       &\P \left[-\GminusX < \delta                        \cond[\middle](\GminusX \leq V^{1,i+1} < 0)               \wedge F_{i,1} = 1\right] &&\cdot \P \left[  \GminusX   \leq V^{1,i+1}         < 0              \cond[\middle] F_{i,1} = 1\right]\\
       + \,       &\P \left[\GminusX - 2 \cdot V^{1,i+1} < \delta   \cond[\middle](V^{1,i+1} \leq \GminusX < 0)               \wedge F_{i,1} = 1\right] &&\cdot \P \left[  V^{1,i+1}     \leq \GminusX       < 0              \cond[\middle] F_{i,1} = 1\right]\\
       + \,       &\P \left[\GminusX - 2 \cdot V^{1,i+1} < \delta   \cond[\middle](\GminusX \geq 0 \geq V^{1,i+1})            \wedge F_{i,1} = 1\right] &&\cdot \P \left[  \GminusX   \geq 0 \geq V^{1,i+1}                   \cond[\middle] F_{i,1} = 1\right]\\
       + \,       &\P \left[2 \cdot V^{1,i+1} - \GminusX < \delta   \cond[\middle](\GminusX \leq 0 \leq V^{1,i+1})            \wedge F_{i,1} = 1\right] &&\cdot \P \left[  \GminusX   \leq 0 \leq V^{1,i+1}                   \cond[\middle] F_{i,1} = 1\right]\B
     \end{alignedat}\\
    &\begin{aligned}
      = \frac{1}{2} &\cdot 
          \Bigg(\frac{1}{2} \cdot \frac{\delta - |\GminusX|}{2 \cdot \delta} + 1 \cdot \frac{|\GminusX|}{2 \cdot\delta} +  1 \cdot \frac{|\GminusX|}{2 \cdot\delta}
          + \frac{1}{2} \cdot \frac{\delta - |\GminusX|}{2 \cdot \delta} + \frac{\delta -|\GminusX|}{2\cdot\delta} \cdot \frac{1}{2} + \frac{\delta - |\GminusX|}{2\cdot\delta} \cdot \frac{1}{2}\Bigg)
          &\\
      = \frac{1}{2} &&\qedhere
    \end{aligned}
  \end{align*}
\end{proof}

With this probability we can now analyze the length of the forced phase.
\begin{lem}
 Given $F_{i,n,d}=1$, define $Y$ as the number of consecutive particle steps, during which dimension $d$ of the respective particle is forced, starting with particle $n$ at iteration $i$. Then for $k\ge 1$ holds $\P[Y=k]=\frac{1}{2^{k-1}}$. \label{lem:row}
\end{lem}
\begin{proof}
Again, we omit the dimension $d$ from the variables.

  By Lemma \ref{lem:stop}, we know that if for one particle $|V^{n,i+1}| + |G - X^{n,i+1}| \geq \delta$ holds, the next particle cannot be
forced and with Lemma \ref{lem:prob}, $\P [F_{i,n+1} = 1 | F_{i,n} = 1] = \frac{1}{2}$.
  By assumption, $\P[F_{i,n}=1] = 1$.
  By induction on $k$, we have:
  \begin{align*}
    \P[Y=2]=&\P\left[\bigwedge_{v=0}^{2-1} F_{i+\lfloor\frac{n+v}{N}\rfloor,n+v\bmod N}=1\right]\\
           =&\P\left[F_{i+\lfloor\frac{n+1}{N}\rfloor,n+1 \bmod N}=1 \cond[\middle] F_{i,n}=1\right] \cdot \P\left[F_{i,n}=1\right]\\
           =&\P\left[F_{i+\lfloor\frac{n+1}{N}\rfloor,n+1\bmod N}=1 \cond[\middle] F_{i,n}=1\right] = \frac{1}{2} = \frac{1}{2^{2-1}}\\
    \P[Y=k]=&\P\left[\bigwedge_{v=0}^{k-1} F_{i+\lfloor\frac{n+v}{N}\rfloor,n+v\bmod N}=1\right]\\
    =&\P\left[F_{i+\lfloor\frac{n+(k-1)}{N}\rfloor,n+(k-1)\bmod N}=1 \cond[\middle] \bigwedge_{v=0}^{k-2} F_{i+\lfloor\frac{n+v}{N}\rfloor,n+v\bmod N}=1\right]\\
     &\cdot \P\left[\bigwedge_{v=0}^{k-2} F_{i+\lfloor\frac{n+v}{N}\rfloor,n+v\bmod N}=1\right]\\
           = &\,\frac{1}{2} \cdot \frac{1}{2^{k-2}} = \frac{1}{2^{k-1}}
  \end{align*}
\end{proof}
Finally for the recovery phase we cannot give an exact length, as it depends on the positions of the particles in the search space.
The following Lemma~\ref{lem:first_unforced} and Fig.~\ref{fig:dist} however give us at least an idea of the behavior of the particle during this phase. 
\begin{lem}
  The probability for a particle $n$ to get a partial potential $\phi_d^n$ greater than $\delta$ when not forced for the first time after a chain of forced moves is
$$\P\left[ 
\phi^{n,i+1} \geq \delta \cond[\middle] F_{i,n,d} = 0 \wedge F_{i-1,n,d} = 1\right] \geq \frac{1}{2}\cdot\Big(1 - \frac{1}{2\chi}\Big)\enspace.
$$\label{lem:first_unforced}
\end{lem}

\begin{proof}
  By assumption, $L^n=G$.
  Again, $d$ is omitted from the variables.\\
  Applying the PSO movement equations leads to
  \interdisplaylinepenalty=10000
  \begin{align*}
    \phi^{n,i+1} &= |V^{n,i+1}| + |G - X^{n,i+1}| = |V^{n,i+1}| + |\smash{\overbrace{G - X^{n,i}}^{=:\GminusX}} - V^{n,i+1}|\\
    &= |\chi \cdot V^{n,i} + (c_1 \cdot r+ c_2 \cdot s)\cdot\GminusX 
                                                                 | + |\GminusX 
                                                                            - \chi \cdot V^{n,i} - (c_1 \cdot r+ c_2 \cdot s)
\cdot\Delta|\\ 
    &= |\chi \cdot V^{n,i} + (c_1 \cdot r+ c_2 \cdot s)\cdot\GminusX| + |- \chi \cdot V^{n,i} - ((c_1 \cdot r+ c_2 \cdot s) - 1)\cdot\GminusX|\\
    &= |\chi \cdot V^{n,i} + (c_1 \cdot r+ c_2 \cdot s)\cdot\GminusX| + |\chi \cdot V^{n,i} + ((c_1 \cdot r+ c_2 \cdot s) - 1)\cdot\GminusX|\\
    &\geq |\chi \cdot V^{n,i} + (c_1 \cdot r+ c_2 \cdot s)\cdot\GminusX + \chi \cdot V^{n,i} + ((c_1 \cdot r+ c_2 \cdot s) - 1)\cdot\GminusX|\\
    &= |2 \cdot \chi \cdot V^{n,i} + (2\cdot (c_1 \cdot r+ c_2 \cdot s) - 1)\cdot\GminusX|
  \end{align*}
    \allowdisplaybreaks
Hence,
  \begin{align*}
    \MoveEqLeft[1] \P\left[\phi^{n,i+1} 
            \geq \delta \cond[\middle] F_{i,n} = 0 \wedge F_{i-1,n} = 1\right] \\
    &\geq \P\left[|2\chi \cdot V^{n,i} + (2\cdot (c_1 \cdot r+ c_2 \cdot s) - 1)\cdot\GminusX| \geq \delta \cond[\middle] F_{i-1,n} = 1\right]\\
    &=\P\left[|2\chi \cdot V^{n,i} + (2\cdot (c_1 \cdot r+ c_2 \cdot s) - 1)\cdot\GminusX| \geq \delta \cond[\middle]
       \begin{aligned}
         &V^{n,i} \cdot (2\cdot (c_1 \cdot r+ c_2 \cdot s) - 1)\cdot\GminusX \geq 0 \\
         &\wedge F_{i-1,n} = 1
       \end{aligned}
           \right]\\
    &\qquad\cdot\P\left[V^{n,i}\cdot(2\cdot (c_1 \cdot r+ c_2 \cdot s) - 1)\cdot\GminusX \geq 0 \cond[\middle] F_{i-1,n} = 1\right]\\
    &\quad\!+\P\left[|2\chi \cdot V^{n,i} + (2\cdot (c_1 \cdot r+ c_2 \cdot s) - 1)\cdot\GminusX| \geq \delta \cond[\middle]
       \begin{aligned}
         &V^{n,i}\cdot(2\cdot (c_1 \cdot r+ c_2 \cdot s) - 1)\cdot(G-x^{n,i}) < 0 \\
         &\wedge F_{i-1,n} = 1
       \end{aligned}
       \right]\\
    &\qquad\cdot\P\left[V^{n,i}\cdot(2\cdot (c_1 \cdot r+ c_2 \cdot s) - 1)\cdot\GminusX < 0 \cond[\middle] F_{i-1,n} = 1\right]\\
    &\geq \P\left[|2\chi \cdot V^{n,i}| \geq \delta \cond[\middle] V^{n,i}\cdot(2\cdot (c_1 \cdot r+ c_2 \cdot s) - 1)\cdot\GminusX \geq 0 \wedge F_{i-1,n} = 1\right]\\
         &\qquad\cdot\P\left[V^{n,i}\cdot(2\cdot (c_1 \cdot r+ c_2 \cdot s) - 1)\cdot\GminusX \geq 0 \cond[\middle] F_{i-1,n} = 1\right]\\
    &\geq\underbrace{\P\left[2\chi \cdot|V^{n,i}| \geq \delta \cond[\middle] F_{i-1,n} = 1\right]}_{=:\alpha} \cdot
           \underbrace{\P\left[V^{n,i}\cdot(2\cdot (c_1 \cdot r+ c_2 \cdot s) - 1)\cdot\GminusX \geq 0\right]}_{=:\beta}\geq\Big(1 - \frac{1}{2\chi}\Big)\cdot\frac{1}{2}
  \end{align*}
  Under the assumption $F_{i-1,n} = 1$, we have
  $V^{n,i}\sim\mathcal{U}(-\delta,\delta)$ and
  $|V^{n,i}| \sim \mathcal{U}(0,\delta)$ ($\mathcal{U}$ denotes the
  continuous uniform distribution).
  Therefore we get
  $\alpha\geq 1-\frac{1}{2\chi}$ and can derive $\beta = \frac{1}{2}$ by the following equations.
  \interdisplaylinepenalty=10000
  \begin{align*}
    \MoveEqLeft \P\left[V^{n,i}\cdot(2\cdot (c_1 \cdot r+ c_2 \cdot s) - 1)\cdot\GminusX \geq 0\right]\\
    &=\P\left[V^{n,i}\cdot(2\cdot (c_1 \cdot r+ c_2 \cdot s) - 1)\cdot\GminusX \geq 0 \cond[\middle] V^{n,i} \geq 0\right ] \cdot \P\left[ V^{n,i} \geq 0 \right]\\
    &\qquad+\P\left[V^{n,i}\cdot(2\cdot (c_1 \cdot r+ c_2 \cdot s) - 1)\cdot\GminusX \geq 0 \cond[\middle] V^{n,i} < 0\right ] \cdot \P\left[V^{n,i} < 0\right]\\
    &=\frac{1}{2}\cdot \left(\P\left[(2\cdot (c_1 \cdot r+ c_2 \cdot s) - 1)\cdot\GminusX \geq 0\right] + \P\left[(2\cdot (c_1 \cdot r+ c_2 \cdot s) - 1)\cdot\GminusX < 0\right]\right)\\
    &=\frac{1}{2}
  \end{align*}
  \allowdisplaybreaks
\end{proof}
Note that the lower bound on the probability in Lemma~\ref{lem:first_unforced} is negative for
$\chi < \frac{1}{2}$. The actual probability depends on the position in the search space. With 
$\chi > \frac{1}{2}$, the probability for a particle $n$ to get partial potential $\phi^n>\delta$ is positive independent of the position of the particle. Therefore, the length of the recovery phase will decrease drastically 
with smaller $\chi$ values.
\begin{figure}
  \centering
  \scalebox{.9}{\input{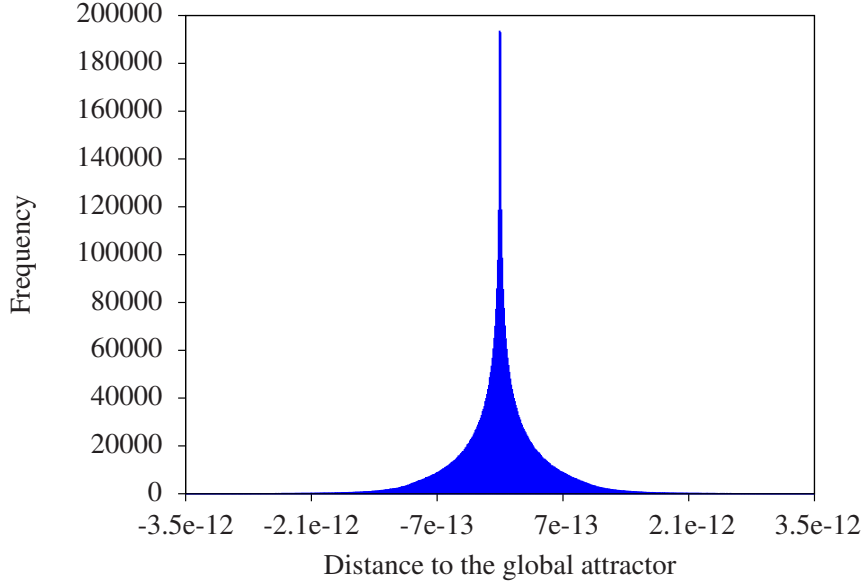}}
  \caption{Frequency of the distance to the global attractor of one particle in one dimension each iteration for $5\,000\,000$ iterations in a run with $5$ particles and $15$ dimensions and $\delta = 10^{-7}$ initialized at the global optimum.}
  \label{fig:dist}
\end{figure}
We may
conclude that for a given dimension we get the following behavior in the pulsation phase.

\begin{sat}
  \label{sat:stag}
  Let $\cal S$ be a swarm of particles at a (local) optimum of an arbitrary function $f$, i.\,e., $\forall n \in {\cal S}: L^n = G = {\rm Opt}(f)$ with ${\rm Opt}(f)$ being a (local) optimum of the function $f$.
  The swarm repeats three phases independently for each dimension:
\begin{itemize}
\item the forced phase of length $k \ge 1$ with probability $\frac{1}{2^{k-1}}$,
\item the lockout phase of length $N$,
\item and the recovery phase of length $\ell \ge 0$
\end{itemize}
  leading to a stagnation of $\sigma(I)/|I|\to\partial\sigma$ at a (local) optimum depending on $N$ and $D$.
\end{sat}

Starting at one forced particle, the chance for the next $k$ particles to be forced has probability $2^{-k+1}$
as stated in Lemma \ref{lem:row} which is independent of the number of particles. At some particle $u$,
this sequence of forced particles ends, and in this moment the forced phase ends.
With Lemma \ref{lem:stop}, at least the following $N$ particles will not be forced
which constitutes the lockout phase.
The next particle that can be forced is $(u+1)\bmod N=v$,
and for that all particles $n$ that move between $u$ and $v$ need to keep partial potential
$\phi_d^n<\delta$.
All these particles use the regular swarm movement equations from Def.~\ref{modified}, and,
with $G$ being constant, they are independent of each other.
As shown in Lemma~\ref{lem:first_unforced}, the chance for a particle $n$ to get partial potential
$\phi^{n}_d>\delta$ can be quite probable
depending on the actual swarm parameters. Hence, with increasing $N$ 
the chance for a particle $n$ to get partial potential $\phi_d^n > \delta$ increases exponentially
in the order of the probability of the complementary event.
The exact probability for this event depends on the position of the particle relative to the global attractor.
The same is
true for the number of
iterations a particle needs to fix its partial potential $\phi$, but with normal swarm behavior the particle starts to converge to the global attractor and to decrease the potential. Fig.~\ref{fig:dist} shows the distance of the particles to the global attractor at the global optimum. We can see that the particles tend to converge to the attractor with high probability thus ending the recovery phase. Therefore, the 
particles once again start to get forced and the same pattern repeats itself.


\section{Termination criteria}
\label{sec:term}

In Sec.~\ref{sec:fm} above, we discussed the behavior of the particle swarm
when the global and local attractors are directly at a local optimum.
In \cite{SW:15}, it is shown that with f-PSO the
global attractor converges (in the infinite-time sense) to an
(at least local) optimum almost surely.
However, this convergence can be very slow.
Therefore, the particle swarm should be stopped at some time
and either be re-started with, when indicated, different parameter settings, or,
if the current global attractor
does not satisfy the desired quality, to use a different algorithm.
As
the swarm does not know when it is
close to a local optimum
and due to the forced steps,
the swarm will just converge up to a $\delta$-neighborhood of the global attractor
which in turn approached a (local) optimum.
With our observations from Sec.~\ref{sec:fm} this means that
the forcing frequency will exceed a fixed value
which can be measured.

In the following, we will define a new stopping criterion that stops the algorithm
at time when the f-PSO algorithm reaches this frequency.
Note that reaching this forcing frequency is just an indication that
the swarm is close to an optimum.

\subsection{Full stagnation criterion}
\label{sec:term1}

By Theorem~\ref{sat:stag}, we know that the absolute and relative forcing frequency become fixed if the swarm
would be at a (local) optimum, independent of the objective function.
So one can use an arbitrary known function with known optimum to experimentally determine this frequency by
starting the swarm in the optimum, as we did in Sec.~\ref{subsec:ExpSphere} for \Sphere.

We will use this observation to describe the
first simple termination criterion:
we just stop the swarm when the measured absolute forcing frequency begins to stagnate
and is close to the fixed forcing frequency around a (local) optimum.

\begin{defi}\label{def:term1}
  Let $\sigma_{\rm stag}(N,D,\mu)$ be the objective function-independent, fixed absolute forcing frequency
  around an optimum when intervals are considered with interval length $\mu$,
  and let $\gamma\in\mathbb{N}$.
  The criterion ``$(\gamma,\mu)$-{\sc full stop}'' is to terminate the {\rm f-PSO} execution
  iff $\sigma_{\rm stag}(N,D,\mu) - \sigma(I) \le \gamma$
for some interval $I$ with $|I|=\mu$ .
\end{defi}
We tested $(\gamma,\mu)$-{\sc full stop}
in Line~\ref{line-sixteen} of Algorithm~\ref{alg:modPSO}, i.\,e., f-PSO,
on our benchmark functions with $N=5$ particles, $D=15$ dimensions, $\delta=10^{-7}$ and
interval length $\mu =50\,000$.
The standard swarm parameters were used.
The stagnation frequency $\sigma_{\rm stag}(N,D,\mu) = 318\,350$ was measured on the \Sphere function
with all particles initialized at the global optimum.
We used $\gamma=1350$, so we terminated the execution
when $\sigma(I)\ge 317\,000$ for intervals $I$ of length $\mu$.
If better solutions would be required, one could reduce $\gamma$.
We compared the solutions obtained with f-PSO that used $(\gamma,\mu)$-{\sc full stop} with
the results obtained by f-PSO that run for $15\,000\,000$ iterations.
Table \ref{tab:term1_med} shows the median (not the average) and standard deviation on 500 runs
of the norm of the gradient of the global attractor on these benchmark functions with and without the use of
criterion $(\gamma,\mu)$-{\sc full stop}.
The less the gradient, the better the global attractor.
\begin{table}
  \centering
  \setlength{\extrarowheight}{2pt}
  \begin{tabular}[c]{| l | l | c | c |}
    \hline
    Function $f$ & $\hfill\text{iter}_{\rm term}\hfill$ & $||\nabla f(G_{{\rm iter}_{\rm term}})||$ & $||\nabla f(G_{15\,000\,000})||$ \\
    \hline
    \Sphere & $100\,000 \pm 0$ & $6.65\e{-8} \pm 6.93\e{-9}$ & $3.60\e{-8} \pm 3.83\e{-9}$ \\
    \hline
    \Elliptic  & $300\,000 \pm 154\,324$ & $2.22\e{-5} \pm 1.93\e{-5}$ & $1.17\e{-5} \pm 1.09\e{-5}$\\
    \hline
    \Schwefel & $150\,000 \pm 110\,707$ & $1.94\e{-7} \pm 1.20\e{-7}$ & $1.07\e{-7} \pm 5.96\e{-8}$ \\
    \hline
    \Rastrigin & $100\,000 \pm 5000$ & $1.34\e{-5} \pm 1.61\e{-6}$ & $7.82\e{-6} \pm 1.15\e{-6}$ \\
    \hline
    \Rosenbrock & $850\,000 \pm 490\,538$ & $3.09\e{-5} \pm 6.45\e{-6}$ & $2.04\e{-5} \pm 3.80\e{-6}$ \\
    \hline
  \end{tabular}
\caption{\label{tab:term1_med} Experimental results for $(\gamma,\mu)$-{\sc full stop}.
$\text{iter}_{\rm term}$ denotes the number of iterations (measured every $\mu$th iteration) until
$(\gamma,\mu)$-{\sc full stop} terminates the execution of f-PSO.
$G_t$ denotes the global attractor after $t$ iterations.
The final two columns present the median and (with $\pm$) the standard deviation of the gradient
at the global attractor over $500$ runs of the f-PSO with and without $(\gamma,\mu)$-{\sc full stop}.
In the tests,
$N=5$, $D=15$, $\mu = 50\,000$, $\sigma_{\rm stag}(N,D,\mu) = 318\,350$, $\gamma=1350$,  and $\delta=10^{-7}$.
The less the gradient, the better $G_t$.}
\end{table}
On simple benchmark functions like \Sphere criterion $(\gamma,\mu)$-{\sc full stop} produces sufficiently good output.
We see that the time of termination is significantly earlier than with pre-specified number of iterations
and that the difference in the gradient at termination
is negligible.

\subsection{Partial stagnation criterion}
\label{sec:term2}
On more complex functions like \Rosenbrock and \Rastrigin, we can identify two scenarios
that interfere with the {\sc full stop} termination criterion.

The
first scenario can be observed on the \Rosenbrock function. It may occur that some dimensions are already highly optimized,
i.\,e., close to an optimal coordinate, while others
are still far away from optimum coordinates, a situation which leads to late termination.
With \Rosenbrock when the particles and the attractors are in the valley between
the local and the global attractor, the chance to improve on the not yet optimized dimensions is voided by the worsening in the already good dimensions.
Therefore the swarm will not change the local (and, thus, the global) attractors and the attractors are
far away from each other. In this case the optimized dimensions 
reached the average stagnation frequency $\sigma_{\rm stag}(N,D,\mu)/D$,
but the other dimensions will not get forced much or not at all as shown in Fig.~\ref{fig:opt_steps_ros}.

The second scenario becomes visible on the \Rastrigin function.
This function has many local optima with a global optimum at $\vec 0$.
Unlike to \Rosenbrock there is not much difference in the behavior of the particles in
the different dimensions. Here, the problem originates from many local optima with only a small difference in function values.
This can lead to the global attractor being in a better local optimum than some of the local attractors.
Given this situation the particles have a very small probability to actually improve their local attractor.
At this point, the algorithm should terminate, but given the different positions of the global and the local attractors
these particles cannot decrease their velocity below the necessary bounds to reach the probability of a forced step
at the stagnation frequency. The current $\sigma$ will stay around some value, which
is different from the stagnation frequency of each dimension and with a fluctuation larger
than the fluctuation at the stagnation frequency. However, note that the probability for such an event is quite small and was not encountered
in our experiments.

To account for such situations we formulate an extended termination criterion.
\begin{defi}\label{def:term2}
  Let $\sigma_{\rm stag}(N,D,\mu)$ be as defined in Def.~\ref{def:term1}, and let $\kappa \in [1, D]$.
  The criterion ``$(\kappa,\gamma,\mu)$-{\sc partial stop}'' is to terminate the {\rm f-PSO} execution
  iff $\sigma(I)\ge \kappa\cdot\dfrac{\sigma_{\rm stag}(N,D,\mu)-\gamma} D$
  for some interval $I$ with $|I|=\mu$ .
\end{defi}

Def.~\ref{def:term2} introduces the additional parameter $\kappa$.
It can be interpreted as how many dimensions are required to reach their stagnation frequency such that f-PSO is stopped.
Actually, if $\kappa=D$, we have $(\gamma,\mu)$-{\sc full stop}.
By this parameter, the user can choose the degree of convergence necessary before terminating.
On complex objective functions that take a long time to evaluate, a small $\kappa$ may be favorable leading to
earlier termination and, hence, saving a lot of running time for applying further subsequent optimization methods
to further improve the solution quality. On simple, easily to be evaluated functions such as \Sphere, one may
choose $\kappa=D$, which results in $(\gamma,\mu)$-{\sc full stop}, as already mentioned.
With $(\kappa,\gamma,\mu)$-{\sc partial stop}, both interfering scenarios are accounted for.
There is no differentiation between dimensions, so it makes
no difference that like in the
first scenario some dimensions $d$ are forced such they have reached (or are close to)
their stagnation value $\partial\sigma_d$,
while in other dimensions still almost no forcing takes place.
Similarly, it might happen, like
in the second scenario, that in all dimensions the frequency is less than the stagnation value,
but the overall sum has reached the stagnation value.
We executed a first series of experiments similar to the case with the criterion $(\gamma,\mu)$-{\sc full stop}.
Again, $N=5$ particles, $D=15$ dimensions, $\delta=10^{-7}$, interval length $\mu =50\,000$,
$\sigma_{\rm stag}(N,D,\mu) = 318\,350$, and  $\gamma=1350$.
Tables \ref{tab:term2_med} and \ref{tab:term2_geo} show the results of the experiments.
We tested $(\kappa,\gamma,\mu)$-{\sc partial stop} with $\kappa=2$ and $\kappa=8$,
and for the fixed iteration limit, we used $15\,000\,000$ iterations.
We repeated the experiments 500 times.
\begin{table}
  \setlength{\tabcolsep}{4pt}
   \setlength{\extrarowheight}{2pt}
   \centering
  \begin{tabular}[c]{| l || c | c || c | c || c |}
    \hline
    Function $f$ & $\text{iter}_{\kappa=2}$ & $||\nabla f(G_{\text{iter}_{\kappa=2}})||$ & $\text{iter}_{\kappa=8}$ & $||\nabla f(G_{\text{iter}_{\kappa=8}})||$ & $||\nabla f(G_{15\,000\,000})||$  \B\\
    \hline
    \Sphere & \begin{tabular}{@{}l@{}} $50\,000$ \\$\pm 0$ \end{tabular}
            & \begin{tabular}{@{}l@{}} $7.24\e{-8}$ \\$\pm 7.50\e{-9}$ \end{tabular}
            & \begin{tabular}{@{}l@{}} $50\,000$ \\$\pm 0$ \end{tabular}
            & \begin{tabular}{@{}l@{}} $7.27\e{-8}$ \\$\pm 8.17\e{-9}$ \end{tabular}
            & \begin{tabular}{@{}l@{}} $3.60\e{-8}$ \\$\pm 3.83\e{-9}$ \end{tabular}\\
    \hline
    \Elliptic & \begin{tabular}{@{}l@{}} $50\,000$ \\$\pm 0$ \end{tabular}
              & \begin{tabular}{@{}l@{}} $2.52\e{-5}$ \\$\pm 2.22\e{-5}$ \end{tabular}
              & \begin{tabular}{@{}l@{}} $50\,000$ \\$\pm 0$ \end{tabular}
              & \begin{tabular}{@{}l@{}} $2.44\e{-5}$ \\$\pm 2.24\e{-5}$ \end{tabular}
              & \begin{tabular}{@{}l@{}} $1.17\e{-5}$ \\$\pm 1.09\e{-5}$ \end{tabular}\\
    \hline
    \Schwefel & \begin{tabular}{@{}l@{}} $50\,000$ \\$\pm 0$ \end{tabular}
              & \begin{tabular}{@{}l@{}} $2.22\e{-7}$ \\$\pm 1.24\e{-7}$ \end{tabular}
              & \begin{tabular}{@{}l@{}} $50\,000$ \\$\pm 0$ \end{tabular}
              & \begin{tabular}{@{}l@{}} $2.14\e{-7}$ \\$\pm 1.31\e{-7}$ \end{tabular}
              & \begin{tabular}{@{}l@{}} $1.07\e{-7}$ \\$\pm 5.96\e{-8}$ \end{tabular}\\
    \hline
    \Rastrigin & \begin{tabular}{@{}l@{}} $50\,000$ \\$\pm 0$ \end{tabular}
               & \begin{tabular}{@{}l@{}} $1.45\e{-5} $\\$\pm 1.65\e{-6}$ \end{tabular}
               & \begin{tabular}{@{}l@{}} $50\,000$ \\$\pm 0$ \end{tabular}
               & \begin{tabular}{@{}l@{}} $1.46\e{-5}$\\$\pm 1.72\e{-6}$ \end{tabular}
               & \begin{tabular}{@{}l@{}} $7.82\e{-6}$ \\$\pm 1.15\e{-6}$ \end{tabular}\\
    \hline
    \Rosenbrock & \begin{tabular}{@{}l@{}} $100\,000$ \\$\pm 730\,696$ \end{tabular}
                & \begin{tabular}{@{}l@{}} $8.54\e{-5}$ \\$\pm 3.78\e{-5}$ \end{tabular}
                & \begin{tabular}{@{}l@{}} $150\,000$ \\$\pm 469\,614$ \end{tabular}
                & \begin{tabular}{@{}l@{}} $7.02\e{-5}$ \\$\pm 1.80\e{-5}$ \end{tabular}
                & \begin{tabular}{@{}l@{}} $2.04\e{-5}$ \\$\pm 3.80\e{-6}$ \end{tabular}\\
    \hline
  \end{tabular}
\caption{\label{tab:term2_med}Experimental results for $(\kappa,\gamma,\mu)$-{\sc partial stop}
(for notions, see Table~\ref{tab:term1_med}).
\emph{Median} and (with $\pm$) the standard deviation over 500 runs of the stopped f-PSO with
$N=5$, $D=15$, $\mu = 50\,000$, $\sigma_{\rm stag}(N,D,\mu) = 318\,350$, $\gamma=1350$,  and $\delta=10^{-7}$,
and $\kappa=2$ and $\kappa=8$, resp., against an `unstopped' f-PSO terminated after $15\,000\,000$ iterations.}
\end{table}
\begin{table}
  \centering
  \begin{tabular}[c]{| l || c | c || c | c || c |}
    \hline
    Function $f$ & $\text{iter}_{\kappa=2}$ & $||\nabla f(G_{\text{iter}_{\kappa=2}})||$ & $\text{iter}_{\kappa=8}$ & $||\nabla f(G_{\text{iter}_{\kappa=8}})||$ & $||\nabla f(G_{15\,000\,000})||$ \B\\
    \hline
    \Sphere & $50\,000$ & $7.15\e{-8}$ & $50\,000$ & $7.12\e{-8}$ & $3.60\e{-8}$  \\
    \hline
    \Elliptic & $50\,000$ & $1.95\e{-5}$ & $50\,000$ & $1.95\e{-5}$ & $1.17\e{-5}$ \\
    \hline
    \Schwefel & $50\,000$ & $1.04\e{-7}$ & $50\,000$ & $1.07\e{-7}$ & $1.07\e{-7}$ \\
    \hline
    \Rastrigin & $50\,000$ & $1.43\e{-5}$ & $50\,000$ & $1.44\e{-5}$ & $7.82\e{-6}$ \\
    \hline
    \Rosenbrock & $153\,672$ & $8.80\e{-5}$ & $176\,989$ & $6.91\e{-5}$ & $2.04\e{-5}$\\
    \hline
  \end{tabular}
  \caption{\label{tab:term2_geo}Experimental results for $(\kappa,\gamma,\mu)$-{\sc partial stop}
(for notions, see Table~\ref{tab:term1_med}). \emph{Geometric Mean}
over 500  runs of the stopped f-PSO with
$N=5$, $D=15$, $\mu = 50\,000$, $\sigma_{\rm stag}(N,D,\mu) = 318\,350$, $\gamma=1350$,  and $\delta=10^{-7}$,
and $\kappa=2$ and $\kappa=8$, resp., against an `unstopped' f-PSO terminated after $15\,000\,000$ iterations.}
\end{table}
We conclude that in most cases a small value of $\kappa$ is sufficient
without a significant loss
in the quality of the returned solution.
However, the \Rosenbrock measurement suggests that on complex objective
functions, a small value of $\kappa$ might lead to too early termination.

If there are many particles and/or a small number of dimensions,
the neighborhood of a local optimum can be reached way sooner.
Therefore, the question comes up whether the interval length can be reduced without
losing
the quality of $(\kappa,\gamma,\mu)$-{\sc partial stop}.
We repeated our experiments for an interval length of $\mu=5\,000$.
Here, $N=5$ particles, $D=15$ dimensions, $\delta=10^{-7}$
and adjusted, $\sigma_{\rm stag}(N,D,\mu) = 31\,835$, and  $\gamma=135$
were applied.
The results are presented in Tables \ref{tab:term2_med_interval} and \ref{tab:term2_geo_interval}.
\begin{table}
  \centering
  \setlength{\tabcolsep}{4pt}
  \setlength{\extrarowheight}{2pt}
  \begin{tabular}[c]{| l || c | c || c | c || c |}
    \hline
    Function $f$ & $\text{iter}_{\kappa=2}$ & $||\nabla f(G_{\text{iter}_{\kappa=2}})||$ & $\text{iter}_{\kappa=8}$ & $||\nabla f(G_{\text{iter}_{\kappa=8}})||$ & $||\nabla f(G_{15\,000\,000})||$ \B\\
    \hline
    \Sphere & \begin{tabular}{@{}l@{}} $5\,000$ \\$\pm 0$ \end{tabular}
            & \begin{tabular}{@{}l@{}} $1.10\e{-7}$ \\$\pm 1.37\e{-8}$ \end{tabular}
            & \begin{tabular}{@{}l@{}} $10\,000$ \\$\pm 2200$ \end{tabular}
            & \begin{tabular}{@{}l@{}} $9.52\e{-8}$ \\$\pm 1.22\e{-8}$ \end{tabular}
            & \begin{tabular}{@{}l@{}} $3.63\e{-8}$ \\$\pm 3.73\e{-9}$ \end{tabular}
            \\
    \hline
    \Elliptic & \begin{tabular}{@{}l@{}} $5\,000$ \\$\pm 1310$ \end{tabular}
              & \begin{tabular}{@{}l@{}} $1.11\e{-4}$ \\$\pm 1.54\e{-4}$ \end{tabular}
              & \begin{tabular}{@{}l@{}} $10\,000$ \\$\pm 0$ \end{tabular}
              & \begin{tabular}{@{}l@{}} $3.93\e{-5}$ \\$\pm 3.73\e{-5}$ \end{tabular}
              & \begin{tabular}{@{}l@{}} $1.28\e{-5}$ \\$\pm 1.00\e{-5}$ \end{tabular}
              \\
    \hline
    \Schwefel & \begin{tabular}{@{}l@{}} $15\,000$ \\$\pm 2502$ \end{tabular}
              & \begin{tabular}{@{}l@{}} $3.87\e{-7}$ \\$\pm 2.18\e{-7}$ \end{tabular}
              & \begin{tabular}{@{}l@{}} $15\,000$ \\$\pm 1034$ \end{tabular}
              & \begin{tabular}{@{}l@{}} $3.12\e{-7}$ \\$\pm 1.82\e{-7}$ \end{tabular}
              & \begin{tabular}{@{}l@{}} $1.01\e{-7}$ \\$\pm 6.02\e{-8}$ \end{tabular}
              \\
    \hline
    \Rastrigin & \begin{tabular}{@{}l@{}} $5\,000$ \\$\pm 0$ \end{tabular}
               & \begin{tabular}{@{}l@{}} $2.12\e{-5}$ \\$\pm 2.69\e{-6}$ \end{tabular}
               & \begin{tabular}{@{}l@{}} $5\,000$ \\$\pm 2449$ \end{tabular}
               & \begin{tabular}{@{}l@{}} $1.97\e{-5}$ \\$\pm 2.64\e{-6}$ \end{tabular}
               & \begin{tabular}{@{}l@{}} $7.82\e{-6}$ \\$\pm 1.13\e{-6}$ \end{tabular}
               \\
    \hline
    \Rosenbrock & \begin{tabular}{@{}l@{}} $65\,000$ \\$\pm 442\,699$ \end{tabular}
                & \begin{tabular}{@{}l@{}} $1.40\e{-3}$ \\$\pm 1.04\e{-3}$ \end{tabular}
                & \begin{tabular}{@{}l@{}} $80\,000$ \\$\pm 381\,917$ \end{tabular}
                & \begin{tabular}{@{}l@{}} $1.74\e{-4}$ \\$\pm 5.34\e{-5}$ \end{tabular}
                & \begin{tabular}{@{}l@{}} $2.04\e{-5}$ \\$\pm 4.03\e{-6}$ \end{tabular}
                \\
    \hline
  \end{tabular}
  \caption{\label{tab:term2_med_interval}Experimental results for $(\kappa,\gamma,\mu)$-{\sc partial stop} with reduced interval length
(for notions, see Table~\ref{tab:term1_med}).
\emph{Median} and (with $\pm$) the standard deviation over 500 runs of the stopped f-PSO with
$N=5$, $D=15$,
$\pmb{\mu = 5\,000}$,
$\pmb{\sigma}_{\textbf{stag}}\pmb{(N,D,\mu)= 31\,835}$, $\gamma=1350$,  and $\delta=10^{-7}$,
and $\kappa=2$ and $\kappa=8$, resp., against an `unstopped' f-PSO terminated after $15\,000\,000$ iterations.}
\end{table}
\begin{table}
  \centering
  \begin{tabular}[c]{| l || c | c || c | c || c |}
    \hline
    Function $f$ & $\text{iter}_{\kappa=2}$ & $||\nabla f(G_{\text{iter}_{\kappa=2}})||$ & $\text{iter}_{\kappa=8}$ & $||\nabla f(G_{\text{iter}_{\kappa=8}})||$ & $||\nabla f(G_{15\,000\,000})||$ \B\\
    \hline
    \Sphere & $5\,000$ & $1.09\e{-7}$ & $8\,337$ & $9.41\e{-8}$ & $3.58\e{-8}$  \\
    \hline
    \Elliptic & $5\,262$ & $9.68\e{-5}$ & $10\,000$ & $3.15\e{-5}$ & $9.62\e{-6}$ \\
    \hline
    \Schwefel & $12\,276$ & $3.84\e{-7}$ & $15\,150$ & $3.15\e{-7}$ & $1.04\e{-7}$ \\
    \hline
    \Rastrigin & $5\,000$ & $2.09\e{-5}$ & $6\,587$ & $1.96\e{-5}$ & $7.75\e{-6}$ \\
    \hline
    \Rosenbrock & $92\,560$ & $1.41\e{-3}$ & $111\,448$ & $1.80\e{-4}$ & $2.01\e{-5}$\\
    \hline
  \end{tabular}
  \caption{\label{tab:term2_geo_interval}Experimental results for $(\kappa,\gamma,\mu)$-{\sc partial stop} with reduced interval length
(for notions, see Table~\ref{tab:term1_med}).
\emph{Geometric Mean} over 500 runs of the stopped f-PSO with
$N=5$, $D=15$,
$\pmb{\mu = 5\,000}$,
$\pmb{\sigma}_{\textbf{stag}}\pmb{(N,D,\mu)= 31\,835}$, $\gamma=1350$,  and $\delta=10^{-7}$,
and $\kappa=2$ and $\kappa=8$, resp., against an `unstopped' f-PSO terminated after $15\,000\,000$ iterations.}
\end{table}
As presumed for the simple functions,
f-PSO with $(\kappa,\gamma,\mu)$-{\sc partial stop}
terminates significantly earlier with only a small decline in the quality of the returned solution.
For \Rosenbrock, the difference in the quality of the returned solution
with different choices of the parameter $\kappa$ becomes more distinct.
Furthermore it should be noted that as shown in Fig.~\ref{fig:forced_dimensions_interval} the standard deviation increases
drastically with shorter intervals. In practice a small tolerance is used when testing if the stagnation value is reached. Given this standard deviation on small intervals, a tolerance that compensates this
standard deviation can lead to a too early termination when the stagnation value is not reached and the swarm would recover from the stagnation. This leads to a tradeoff between interval length and the
risk of a termination when the swarm has not reached or is close to a local optimum.


\section{Conclusions}
\label{sec:conclusion}

This paper focused on developing stopping criteria for the f-PSO algorithms.
We showed that at an optimum point the swarm behaves like a blind searching algorithm
and that therefore the number of forced steps in a time interval,
the absolute forcing frequency, is larger than an objection function-independent number.
Thus, reaching this frequency it can be presumed that the swarm is close to
an optimum solution, and it can be stopped.

\section*{Acknowledgements}

We would like to thank Alexander Ra\ss{} for fruitful discussions.
Research funded in parts by the School of Engineering of the University of
Erlangen-Nuremberg.


\bibliographystyle{alpha}
\bibliography{literature}   

\end{document}